\documentclass[10pt,twocolumn,letterpaper]{article}

\usepackage{wacv}              

\newcommand{\doop}{\operatorname{do}}

\usepackage{amsthm}
\theoremstyle{plain}
\newtheorem{theorem}{Theorem}
\newtheorem{proposition}[theorem]{Proposition}

\newtheorem{corollary}[theorem]{Corollary}
\theoremstyle{definition}
\newtheorem{definition}{Definition}

\theoremstyle{remark}

\newcommand{\defeq}{\mathrel{\mathop:}=}

\usepackage{tikz}
\usetikzlibrary{arrows.meta, positioning}

\newcommand{\dashleftrightarrow}{%
  \mathrel{\tikz[baseline]{
    \draw[dashed, dash phase=-0pt, {Computer Modern Rightarrow}-{Computer Modern Rightarrow}] 
      (0, 0.58ex) -- (1.55em, 0.58ex);
  }}%
}

\definecolor{wacvblue}{rgb}{0.21,0.49,0.74}
\usepackage[pagebackref,breaklinks,colorlinks,allcolors=wacvblue]{hyperref}

\def\wacvPaperID{2727} 
\def\confName{WACV}
\def\confYear{2027}

\title{Efficient Fairness Auditing Across Guidance Scales in Text-to-Image Diffusion Models via Causal Abstraction}

\author{
Nabila Tasfiha Rahman\\
University of Arkansas\\
Fayetteville, AR, USA\\
{\tt\small nr072@uark.edu}
\and
Rajatsubhra Chakraborty\\
University of North Carolina at Charlotte\\
Charlotte, NC, USA\\
{\tt\small rchakra6@charlotte.edu}
\and
Depeng Xu\\
University of North Carolina at Charlotte\\
Charlotte, NC, USA\\
{\tt\small dxu7@charlotte.edu}
\and
Lu Zhang\\
University of Arkansas\\
Fayetteville, AR, USA\\
{\tt\small lz006@uark.edu}
}

\begin{document}
\maketitle
\begin{abstract}
Fairness auditing of text-to-image diffusion models often requires generating large numbers of images across sampling configurations, making comprehensive evaluation computationally expensive. We propose a causal-abstraction-based audit instrument for efficiently evaluating fairness under interventions on the classifier-free guidance scale. Given a fixed prompt and a target feature function, we represent the diffusion process as a low-level structural causal model and construct a corresponding high-level model over abstract denoising states. We characterize the projected causal structure, establish identifiability of the fairness-relevant interventional query, and provide sufficient conditions under which the high-level model preserves this query. A probabilistic transformer implements the high-level model as an amortized predictor of target-feature distributions across guidance scales. Experiments evaluate distributional fidelity, fairness-query accuracy, and computational efficiency. We present two auditing demonstrations: one using standard Stable Diffusion 1.5 and another using StayFair, a fairness-enhanced Stable Diffusion model, to examine their behavior across guidance scales.
\end{abstract}
    
\section{Introduction}
\label{sec:intro}

Causal abstraction \cite{rubenstein2017causal,beckers2019abstracting,geiger2025causal} provides a principled way to relate a complex low-level system, such as a high-dimensional vision model, to a simpler, semantically meaningful high-level causal model. Constructing a valid abstraction, however, is fundamentally more difficult than learning a compact representation or an accurate predictive surrogate. Semantic compression is typically many-to-one: distinct low-level states that share the same abstract representation may respond differently to an intervention and produce different downstream outcomes. The discarded information can induce additional causal
dependencies, longer-range temporal effects, and latent confounding among the retained variables \cite{xia2025causal}. Consequently, a high-level model may
reproduce observed data yet fail to preserve the causal mechanisms relevant to an intervention. 
Existing research on causal abstraction develops formal relationships between low- and high-level SCMs and methods for analyzing neural representations and learning abstract causal models \cite{geiger2024finding,geiger2025causal,geiger2021causal,xia2024neural}. Projected abstractions extend this framework to lossy mappings, providing conditions under which a specific causal query can be preserved without reproducing every low-level causal mechanism \cite{xia2025causal}.

Algorithmic auditing, which systematically evaluates the behavior of an AI system against specified criteria such as fairness, safety, or accountability \cite{raji2020closing,lam2024framework}, is particularly well suited to this query-specific perspective. An audit asks how an outcome changes under designated interventions and whether the result satisfies a decision criterion. 
 
A causally consistent abstraction, which preserves not only observational similarity but also fidelity to the relevant interventional distribution and resulting audit decision, can therefore serve as an auditing instrument, replacing repeated executions of the original model with inexpensive high-level simulations.

In this paper, we apply the causal abstraction framework to fairness auditing of text-to-image diffusion models. Prior studies have shown that, although these models produce high-quality visual content, their outputs can reproduce demographic stereotypes for ostensibly neutral occupational prompts \cite{luccioni2023stable,cho2023dall}. Moreover, the
classifier-free guidance (CFG) scale \cite{ho2022classifier} can
change the demographic composition of generated images, making
fairness sensitive to inference-time settings \cite{kim2026stayfair}. These findings
motivate systematic fairness audits of both standard diffusion models and fairness-enhanced guidance methods, such as StayFair \cite{kim2026stayfair}, to assess demographic disparities and the effectiveness of mitigation across guidance scales. However, direct auditing is computationally expensive because each combination of prompt, guidance setting, and random seed requires a complete diffusion rollout.

We address this computational challenge using causal abstraction to reduce the cost of repeated computation. For each fixed prompt, our auditing instrument constructs a high-level causal model that treats the CFG scale as an intervention variable and a final demographic attribute as the audit outcome. Specifically, we map intermediate diffusion states to low-dimensional semantic attributes and train a stochastic transformer neural causal model to capture their evolution. The learned model estimates post-interventional attribute distributions and fairness gaps without rerunning the full diffusion process for every audit sample. To the best of our knowledge, this is the first work to operationalize causal abstraction as an efficient computational instrument for fairness auditing of machine learning models.

A key technical challenge is that the abstraction is necessarily lossy. Consequently, discarded low-level information can induce latent confounding among the abstract variables that is not explicitly represented by the learned transformer. As a result, the causal graph induced by the learned high-level model is not generally equivalent to the partially projected graph associated with the underlying diffusion process, raising the question of whether the model preserves the causal fidelity required for auditing. We therefore establish a correctness argument tailored to our auditing instrument: when the guidance-intervention query is identifiable from the high-level observational distribution, accurately learning that distribution suffices to approximate the audit query despite the graph mismatch. 

We evaluate the proposed instrument in terms of distribution
fidelity, audit-decision fidelity, and computational savings. Distribution fidelity is assessed using Wasserstein distance and a bootstrap reference, while decision fidelity compares audit conclusions obtained from the high-level instrument and the original diffusion model. We
present two auditing demonstrations using standard Stable Diffusion
v1.5 and SD1.5 equipped with StayFair, examining demographic
disparities and their sensitivity to guidance interventions, following the indirect-verification paradigm of \cite{lam2024framework}.

Our contributions are summarized as follows:
\begin{itemize}
    \item We introduce a query-specific causal-abstraction methodology
    for constructing efficient auditing instruments for vision
    generative models and instantiate it for fairness auditing of
    text-to-image diffusion.

    \item We establish an instrument-level correctness guarantee that
    recovers an identifiable audit query from an accurately learned
    observational distribution despite lossy abstraction, AIC
    violations, and nonequivalence between the learned and projected
    causal graphs.

    \item We demonstrate the instrument on standard SD1.5 and
    SD1.5+StayFair, auditing how demographic representation and
    fairness decisions vary across classifier-free guidance scales.
\end{itemize}
\section{Background and Related Work}
Throughout this paper, we use uppercase letters to denote random variables and lowercase letters to denote their realizations.
We write $P(a\mid b)$ as shorthand for $P(A=a\mid B=b)$ when there is no ambiguity.

\subsection{Text-to-Image Diffusion Models}

Text-to-image diffusion models generate images conditioned on a text
prompt $c$ through an iterative denoising process. During training, a
clean sample $x_0$ is progressively corrupted according to
\begin{equation}
x_t =
\sqrt{\bar{\alpha}_t}\,x_0
+
\sqrt{1-\bar{\alpha}_t}\,\epsilon,
\qquad
\epsilon \sim \mathcal{N}(0,I),
\end{equation}
where $\bar{\alpha}_t$ is determined by a predefined noise schedule.
A neural network $\epsilon_{\theta}(x_t,t,c)$ learns to predict the
added noise conditioned on the prompt. During inference, the model
starts from $x_T \sim \mathcal{N}(0,I)$ and iteratively denoises
$x_T \rightarrow x_{T-1} \rightarrow \cdots \rightarrow x_0$ to
produce the final sample.

\paragraph{Classifier-free guidance.}

Classifier-free guidance (CFG) \cite{ho2022classifier} avoids the
auxiliary classifier used by earlier conditional samplers
\cite{dhariwal2021diffusion} by combining conditional and unconditional
noise predictions:
\begin{equation}
\tilde{\epsilon}_{\theta}(x_t,t,c)
=
\epsilon_{\theta}(x_t,t,\varnothing)
+
w\big(
\epsilon_{\theta}(x_t,t,c)
-
\epsilon_{\theta}(x_t,t,\varnothing)
\big),
\label{eq:cfg}
\end{equation}
where $\varnothing$ denotes the null prompt and $w \geq 0$ is the
guidance scale. Under an idealized score interpretation, the guided
prediction corresponds locally to a distribution proportional to
$p_{\theta}(x_t)p_{\theta}(c\mid x_t)^w$. Increasing $w$ therefore
strengthens prompt conditioning but can reduce diversity and alter
demographic representation. Because $w$ is selected at inference and
affects every denoising step, we treat it as an intervention variable
in our auditing framework.

\subsection{Fairness in Diffusion Models}

Text-to-image diffusion models inherit the demographic skew of their web-scraped training corpora \cite{schuhmann2022laion} and typically exaggerate it. For
example, occupational prompts yield outputs whose gender and ethnicity distributions exceed real-world proportions \cite{luccioni2023stable,bianchi2023easily}, association tests reveal stereotype and valence effects that distribution counting alone misses \cite{wang2023t2iat}, and the skew in generated images is generally larger than the skew in the training data \cite{seshadri2024bias}.
Existing audits commonly infer demographic attributes using external
classifiers or CLIP-based scorers \cite{radford2021learning} and
compare their distribution against a specified reference. Following
this approach, we audit prompt-specific demographic representation
relative to a prespecified fairness reference, aggregating over
multiple random seeds and accounting for estimation uncertainty.

Demographic disparities can emerge during denoising rather than only
in the final image. Prior work has identified gender-dependent
generation trajectories \cite{wu2024exposed} and interactions between
demographic and semantic concepts across diffusion timesteps
\cite{chakraborty2025biasmap}. Existing mitigation methods therefore
intervene through prompt or guidance modification
\cite{friedrich2023fair,zhang2023itigen}, cross-attention editing
\cite{orgad2023editing,gandikota2024unified}, or model adaptation
\cite{shen2024finetuning,parihar2024balancing}. These intermediate
dynamics motivate our abstraction of diffusion trajectories into
semantically meaningful demographic attributes.

A further challenge is that demographic representation can change
with the classifier-free guidance scale $w$
\cite{friedrich2023fair,seshadri2024bias,kim2026stayfair}. Kim et
al. \cite{kim2026stayfair} distinguish guidance-induced bias from
underlying model bias and show that mitigation achieved at one scale
may not persist at others. Consequently, both standard diffusion
models and fairness-enhanced methods such as StayFair require
evaluation across guidance settings. Direct evaluation, however,
requires a complete diffusion rollout for each combination of
prompt, guidance scale, and random seed. Our auditing instrument
addresses this cost by learning a high-level causal abstraction from
diffusion trajectories and using it to estimate demographic
distributions and fairness decisions under guidance interventions.

\subsection{Principles of Causal Abstraction}
\label{sec:causal-abstraction-principles}

Causal abstraction studies formal relationships between causal models
at different levels of granularity and the conditions under which a
simplified high-level model preserves the interventional behavior of
a detailed low-level model
\cite{rubenstein2017causal,beckers2019abstracting,geiger2021causal,geiger2025causal,xia2024neural,xia2025causal}.

Symbolically, let \(\mathcal{M}_L\) and \(\mathcal{M}_H\) denote the low- and high-level SCMs,
respectively, and let \(\mathcal{X}_L\) and \(\mathcal{X}_H\) denote
their state spaces. An abstraction mapping
\(
    \tau:\mathcal{X}_L\rightarrow\mathcal{X}_H
\)
maps low-level states to their high-level
representations. The mapping is generally many-to-one, allowing the
high-level SCM to discard low-level information that is not relevant
to the target analysis.

Let \(Q\) denote a causal query on \(\mathcal{M}_L\), such as an
interventional distribution, and let \(\tau(Q)\) denote its
corresponding high-level query induced by the mapping. The objective of causal abstraction is
to ensure that evaluating \(Q\) in the low-level model and expressing
its answer through \(\tau\) agrees with evaluating \(\tau(Q)\) directly
in the high-level model, i.e.,
\(
    \tau\!\left(\mathcal{M}_L(Q)\right)
    =
    \mathcal{M}_H\!\left(\tau(Q)\right).
\)
This requirement may be imposed
only for a specified collection of queries, rather than for all
possible interventions. Moreover, because \(\tau\) can be lossy, the
high-level mechanisms may be probabilistic even when the corresponding low-level mechanisms are deterministic. When exact consistency is not attainable, fidelity can instead be assessed using a distributional distance,
such as the Wasserstein distance.

Causal abstraction is well suited to auditing because an
audit is typically defined by a restricted collection of causal
queries. For example, an audit may examine how the distribution of a
target feature changes under interventions on a model configuration. If the high-level model preserves these queries, it can be used in place of the low-level model to estimate the post-interventional distributions required by the audit. The abstraction thereby provides an amortized audit instrument: its construction cost is incurred once, after which it can efficiently answer multiple audit queries. 
\section{Fairness Auditing of Text-to-Image Diffusion Models}
\label{sec:fairness-audit-problem}

Following the technical logic of indirect verification
\cite{lam2024framework}, we distinguish the auditee and auditor from the audited
system. The audited system is the low-level text-to-image diffusion model
$\mathcal{M}_L$, whereas the auditee is the model provider. Instead of requiring the auditor to repeatedly execute the full diffusion model, the auditee performs a prespecified
evidence-generation procedure and submits the resulting technical
artifacts. For a fixed prompt $c$, the supporting evidence $\mathcal{E}^{c}$ includes diffusion trajectories and
their associated prompts, CFG scales, random seeds, sampler
configurations, and model provenance.
The auditor verifies that the submitted evidence conforms to the
prespecified model, prompt, guidance, sampling, and provenance
requirements. The auditor then independently applies the abstraction
maps, trains a high-level causal model on a designated training split,
validates its causal-abstraction fidelity using held-out
evidence, and applies the resulting instrument to the fairness
criterion. Our scope is limited to the technical evaluation procedure; organizational
governance, auditor independence, public reporting, and certification
are outside the scope of this work.

\paragraph{Audited system and intervention.}
For each audit,
we fix a prompt \(c\) according to a prespecified prompt-generation
protocol. For example, a gender-neutral prompt may be
\(
    c=\text{``A portrait photo of a lawyer.''}
\)

Let \(X_T\sim P_X\) denote the initial noise, and let \(U\sim P_U\)
collect any additional randomness introduced by the sampler during
image generation. The CFG scale is an intervenable variable with audit
range \(\mathcal{W}\). We write \(\doop(w)\) as
shorthand for \(\doop(W=w)\). For any \(w\in\mathcal{W}\), the image
generated under \(\doop(w)\) is denoted by
\[
    X_0^w
    =
    \mathcal{M}_L(c,X_T,U;\doop(w)).
\]

An audit considers one or more target features through a functional mapping \(\tau_0:\mathcal{X}\to [0,1]^d\) and obtains \(A_0=\tau_0(X_0)\). For simplicity, and to
facilitate comparison with previous fairness studies of text-to-image
models, we consider a single target feature and set \(d=1\). For
example, \(a_0=\tau_0(x_0)\) may be the probability assigned by an
attribute classifier to perceived masculine presentation in image
\(x_0\).
For prompt \(c\) and CFG scale \(w\), the post-interventional
target-feature distribution is characterized by
\(
    P^{\mathcal{M}_L}
    \left(
        A_0\leq a_0
        \mid c,\doop(w)
    \right),
\)
and the expected target-feature
value is
\[
    \rho(w,c)
    \defeq
    \mathbb{E}^{\mathcal{M}_L}
    \left[
        A_0
        \mid c,\doop(w)
    \right].
\]

\paragraph{Fairness metric.}
Let \(\rho^\star(c)\in[0,1]\) denote the prespecified fairness reference
for prompt \(c\). Under demographic representation parity, for example,
one may set \(\rho^\star(c)=1/2\), and the prompt-specific fairness gap is
\(
    \phi(w,c)
    =
    \left|
        \rho(w,c)-\rho^\star(c)
    \right|.
\)
Given a prespecified set of guidance scales $\mathcal{W}_K=\{w_1,\ldots,w_K\}\subseteq\mathcal{W}$, the formal
decision target of the audit is therefore
\begin{equation}
    \Phi_{K}(c)
    \defeq
    \max_{w_j\in\mathcal{W}_K}
    \phi(w_j,c).
    \label{eq:low-level-audit-target}
\end{equation}

\begin{definition}[Feature-based fairness audit]
\label{def:feature-based-audit}

Given an audited diffusion model $\mathcal{M}_L$, auditee-provided
evidence $\mathcal{E}^{c}$, a target feature function
$\tau_0:\mathcal{X}\rightarrow[0,1]$, audited guidance settings
$\mathcal{W}_K$, fairness reference $\rho^\star(c)$, and tolerance
$\delta\geq0$, the feature-based fairness audit is the following procedure:
\begin{enumerate}
    \item construct a high-level causal model from a designated training
    split of $\mathcal{E}^{c}$;
    \item validate its fidelity using held-out evidence;
    and
    \item use the validated high-level model to estimate
    $\Phi_{K}(c)$ and determine whether the fairness criterion is
    supported, violated, or inconclusive.
\end{enumerate}
\end{definition}

\paragraph{Monte Carlo evaluation.}
We use finite Monte Carlo samples to estimate the target-feature distribution at each predefined CFG scale. Given prompt \(c\), we draw
\(N\) independent pairs
\((X_{T,i},U_i)\sim P_XP_U\) and reuse the same pairs across CFG
scales. Let
\(
    A_{0,i}^{w_j}
    =
    \tau_0\!\left(
        \mathcal{M}_L(
            c,X_{T,i},U_i;\doop(w_j)
        )
    \right)
\)
denote the target feature obtained from the \(i\)-th sample at CFG
scale \(w_j\). We estimate the expected target feature and fairness
gap as
\(
    \hat{\rho}(w_j,c)
    =
    \frac{1}{N}
    \sum_{i=1}^N A_{0,i}^{w_j},
\)
and
\(
    \hat{\phi}(w_j,c)
    =
    \left|
        \hat{\rho}(w_j,c)-\rho^\star(c)
    \right|.
\)
Define the estimator of the worst-case fairness gap as
\(
    \hat{\Phi}_K
    \defeq
    \max_{w_j\in\mathcal{W}_K}\hat{\phi}(w_j,c).
\)
Because \(A_{0,i}^{w_j}\in[0,1]\), Hoeffding's inequality and a union
bound over the \(K\) CFG scales imply that, with probability at least
\(1-\alpha\),
\[
    \Phi_K
    \in
    \left[
        \max\{0,\hat{\Phi}_K-\epsilon_N\},
        \hat{\Phi}_K+\epsilon_N
    \right],
\]
where
\(
    \epsilon_N
    =
    \sqrt{
        \frac{\log(2K/\alpha)}{2N}
    }.
\)
This interval accounts for Monte Carlo uncertainty and covers the
worst-case fairness gap over the predefined scale set
\(\mathcal{W}_K\).

\paragraph{Audit conclusion.}
Let \(L_{\mathrm{audit}}\) and \(U_{\mathrm{audit}}\) denote the lower
and upper endpoints of the interval above. For a prespecified tolerance
\(\delta\), the fairness criterion is supported within the audit scope
if \(U_{\mathrm{audit}}\leq\delta\). A fairness violation is detected
if \(L_{\mathrm{audit}}>\delta\); otherwise, the audit is inconclusive.

\section{Causal-Abstraction Audit Instrument}
In this section, we construct a high-level causal model from the submitted trajectory evidence and use it as an audit instrument to evaluate the prespecified fairness query without repeatedly executing the low-level diffusion model. A valid audit instrument must preserve the post-interventional target-feature distribution relevant to the audit and enable this distribution to be identified and efficiently estimated despite information loss introduced by the abstraction. We first formalize the required query-specific consistency between the low- and high-level SCMs. We then construct a high-level causal structure under which the audit query is identifiable. Finally, we instantiate the audit instrument as a transformer neural causal model trained on abstracted diffusion trajectories for amortized inference.

\subsection{Query-Specific Causal Abstraction}
Let the text-to-image diffusion
model be defined as a low-level SCM \(\mathcal{M}_L\) with endogenous variables
\(\mathcal{V}_L\), including the initial noise \(X_T\), the generated
image \(X_0\), the CFG scale \(W\), and a fixed prompt \(c\). 
We consider a
high-level SCM \(\mathcal{M}_H\) with endogenous variables
\(\mathcal{V}_H\), together with an abstraction mapping
\(
    \tau:
    \operatorname{Dom}(\mathcal{V}_L)
    \to
    \operatorname{Dom}(\mathcal{V}_H).
\)
The components of \(\tau\) required by the audit query include the
prespecified target feature function \(\tau_0\), which maps \(X_0\) to
\(A_0=\tau_0(X_0)\). In addition, we also consider a mapping $\tau_T$ that maps \(X_T\) to \(A_T=\tau_T(X_T)\), as well as an identity mapping \(\tau_W(w)=w\) that retains the CFG scale without coarsening. Accordingly, \(\mathcal{V}_H\) contains at least \(W\), \(A_T\), and \(A_0\), but
may contain additional high-level variables. The query-specific framework does not impose a particular choice of these additional variables or their
component mappings. 

Given the high-level SCM $\mathcal{M}_H$ defined above, we present the condition under which $\mathcal{M}_H$ can recover the audit query. For a fixed prompt \(c\) and CFG scale \(w\), define the conditional causal query \(Q_{w,c}\) on the low-level model as
\begin{equation}
\begin{split}
& Q_{w,c}(a_0\mid x_T)
\defeq 
P^{\mathcal{M}_L}
\left(
    A_0\leq a_0
    \mid x_T,c,\doop(w)
\right)\\
& = 
P_{U\sim P_U}
\left(
    \tau_0\!\left(
        \mathcal{M}_L(c,x_T,U;\doop(w))
    \right)
    \leq a_0
\right).
\end{split}
\label{eq:low-level-query}
\end{equation}
That is, \(Q_{w,c}(a_0\mid x_T)\) is the probability that the target
feature of the generated image is at most \(a_0\), conditional on the
prompt \(c\), initial noise \(x_T\), and CFG intervention \(\doop(w)\).
Correspondingly, define the mapping of \(Q_{w,c}\) onto the high-level
SCM \(\mathcal{M}_H\) as
\begin{equation}
\begin{split}
& \tau(Q_{w,c})(a_0\mid x_T)
\defeq 
P^{\mathcal{M}_H}
\left(
    A_0\leq a_0
    \mid \tau_T(x_T),c,\doop(w)
\right)\\
& =
P_{U\sim P_U}
\left(
    \mathcal{M}_H \left(
        c,\tau_T(x_T),U;\doop(w)
    \right)
    \leq a_0
\right).
\end{split}
\label{eq:high-level-query}
\end{equation}
We say that the high-level SCM is \(Q_{w,c}\)-\(\tau\) consistent
with the low-level SCM if the following condition holds.

\begin{definition}[\(Q_{w,c}\)-\(\tau\) consistency]
\label{def:q-tau-consistency}
The high-level SCM \(\mathcal{M}_H\) is
\(Q_{w,c}\)-\(\tau\) consistent with the low-level SCM
\(\mathcal{M}_L\) if for every
\(a_0\in\mathcal{A}_0\), \(a_T\in\mathcal{A}_T\),
\(w\in\mathcal{W}\), and any \(x_T\) satisfying
\(\tau_T(x_T)=a_T\), we have
\begin{equation}
\begin{split}
&\mathbb{E}_{X_T\mid\tau_T(X_T)=a_T}
\left[
    Q_{w,c}(a_0\mid X_T)
\right] =
\tau(Q_{w,c})(a_0\mid x_T).
\end{split}
\label{eq:q-tau-consistency}
\end{equation}
\end{definition}

\begin{proposition}
\label{prop:audit-query-recovery}
If the high-level SCM \(\mathcal{M}_H\) is \(Q_{w,c}\)-\(\tau\) consistent with
\(\mathcal{M}_L\), then it recovers the post-interventional target-feature distribution of the low-level SCM:
\begin{equation}
\begin{split}
&P^{\mathcal{M}_L}
\left(
    \tau_0(X_0)\leq a_0
    \mid c,\doop(w)
\right) = \\
& \qquad P^{\mathcal{M}_H}
\left(
    A_0\leq a_0
    \mid c,\doop(w)
\right)
\end{split}
\label{eq:audit-distribution-recovery}
\end{equation}
for every \(a_0\in\mathcal{A}_0\) and \(w\in\mathcal{W}\).
Consequently, \(\mathcal{M}_H\) also recovers the audit query
\(\rho(w,c)\).
\end{proposition}

\subsection{High-Level Causal Structure}
\label{sec:high-level-structure}

The high-level SCM $\mathcal{M}_H$ and the corresponding
abstraction mapping $\tau$ are designed to satisfy two requirements:
(i) $\mathcal{M}_H$ is $Q_{w,c}$-$\tau$ consistent with the low-level
SCM $\mathcal{M}_L$; and (ii) the high-level query $\tau(Q_{w,c})$ is identifiable from the observational high-level distribution. Condition (ii) ensures that the high-level query is uniquely determined by the observational high-level distribution.
In the following, we present one construction satisfying these requirements, while noting that it is not the only possible
construction.

\paragraph{Variable-wise abstraction.}
For the fixed prompt $c$, we define the endogenous variables of the
low-level SCM as
\(
    \mathcal{V}_L
    =
    \{W,X_T,X_{T-1},\ldots,X_0\},
\)
where $X_t$ for each $t\in\{1,\ldots,T-1\}$ is the diffusion state at step \(t\). The sampling randomness $U$ is treated as exogenous and not included in $\mathcal{V}_L$. 
For each $t$, let
\(
    \tau_t:\mathcal{X} \to \mathcal{A}_t \subseteq[0,1]
\)
be a feature mapping, and define the corresponding high-level
variable as $A_t=\tau_t(X_t)$. 
Together, these component mappings define the abstraction mapping
$\tau$ from the low-level variable domain to the high-level variable
domain. The endogenous variables of the high-level SCM are
\(
    \mathcal{V}_H
    =
    \{W,A_T,A_{T-1},\ldots,A_0\}.
\)

\begin{proposition}
\label{prop:constructive-abstraction}
The mapping $\tau$ defined above is a constructive abstraction
function \cite{xia2025causal} from $\mathcal{V}_L$ to
$\mathcal{V}_H$.
\end{proposition}

Proposition~\ref{prop:constructive-abstraction} establishes that
$\tau$ forms a structurally valid variable-wise abstraction from
$\mathcal{V}_L$ to $\mathcal{V}_H$. It therefore supports the
interpretation of $\mathcal{M}_H$ as a causal abstraction of the
diffusion process. 

\paragraph{Partially projected C-DAG.}
We construct the causal graph of $\mathcal{M}_H$ using the partially
projected C-DAG of \cite{xia2025causal}. The construction begins with
the C-DAG obtained by clustering the low-level variables according to
$\tau$. Let $\mathcal{G}=(\mathcal{V}_H,
\mathcal{E})$ denote this C-DAG, and let
$\mathcal{V}_H^\dagger\subseteq\mathcal{V}_H$ denote the set of
high-level variables whose mappings violate the Abstract Invariance
Condition (AIC). A variable belongs to
$\mathcal{V}_H^\dagger$ when its mapping merges low-level values that
have different causal effects on a downstream variable.

The partially projected C-DAG
$\mathcal{G}^\dagger
=(\mathcal{V}_H,\mathcal{E}^\dagger)$ is initialized
with
$\mathcal{E}^\dagger
=\mathcal{E}$. For every
$X\in\mathcal{V}_H^\dagger$, the following edge-projection rules are
then applied:
\[
\begin{aligned}
    Z\rightarrow X\rightarrow Y
    &\quad\Longrightarrow\quad \textrm{add } Z\rightarrow Y,\\
    Z\dashleftrightarrow X\rightarrow Y
    &\quad\Longrightarrow\quad
        \textrm{add } Z\dashleftrightarrow Y
        \ \text{and}\ 
        X\dashleftrightarrow Y,\\
    Z\leftarrow X\rightarrow Y
    &\quad\Longrightarrow\quad \textrm{add } Z\dashleftrightarrow Y.
\end{aligned}
\]
The rules are applied iteratively until no additional edges are
introduced.

Conceptually, this construction retains the information discarded by
an AIC-violating feature mapping as latent information. Marginalizing
this information can induce additional directed dependencies and
unobserved confounding among the retained high-level variables. The
additional edges ensure that the high-level graph does not impose
conditional-independence or causal-exclusion constraints that are
invalid under the lossy mapping. 

\paragraph{Projected graph for the diffusion process.}
By omitting the fixed prompt \(c\) from the graph, the directed part of the low-level diffusion graph is
\[
    X_T\rightarrow X_{T-1}\rightarrow\cdots\rightarrow X_0
    \quad \textrm{and} \quad 
    W\rightarrow X_t
\]
for each \(t\in\{0,\ldots,T-1\}\)
because each denoising transition depends on both the current
diffusion state and the CFG scale. Since each intervariable cluster
contains a single diffusion state, the corresponding C-DAG \(\mathcal{G}\) has edges
\[
    \{A_t\rightarrow A_{t-1}:1\leq t\leq T\}
    \cup
    \{W\rightarrow A_t:0\leq t<T\}.
\]
Without loss of generality, we treat every nonterminal state mapping as potentially AIC-violating, so that
\(
    \mathcal{V}_H^\dagger
    =
    \{A_T,A_{T-1},\ldots,A_1\}.
\)
Then, applying the partially projected C-DAG construction adds edges
whenever a mapped state belongs to
$\mathcal{V}_H^\dagger$. For example, since $A_t\in\mathcal{V}_H^\dagger$, the path $A_{t+1}\rightarrow A_t\rightarrow A_{t-1}$ induces the additional
edge $A_{t+1}\rightarrow A_{t-1}$. Repeated application of this rule introduces dependencies between a diffusion state and all downstream states. The latent information discarded by $\tau_t$ acts as a hidden confounder of these downstream states, producing the
bidirected edges prescribed by the remaining projection rules. Consequently, the closure of the projection rules adds to the partially projected C-DAG \(\mathcal{G}^{\dagger}\) the directed edges
\[
    \{W\rightarrow A_t:0\leq t<T\}
    \cup
    \{A_s\rightarrow A_t:0\leq t<s\leq T\},
\]
together with the bidirected edges
\[
    \{A_s\dashleftrightarrow A_t:0\leq t<s<T\}.
\]
The final graph is shown in Figure~\ref{fig:partially-projected-cdag}.

\begin{figure}[t]
\centering
\resizebox{0.82\columnwidth}{!}{%
\begin{tikzpicture}[
    >=latex,
    thick,
    state/.style={circle, draw, very thick, minimum size=1.cm, inner sep=0pt, font=\normalsize}
]

\node[state] (AT) at (0,0) {$A_T$};
\node[state] (AT1) at (2.2,0) {$A_{T-1}$};
\node[font=\Large] (dots) at (4.1,0) {$\dots$};
\node[state] (A1) at (6,0) {$A_1$};
\node[state] (A0) at (8.2,0) {$A_0$};

\node[state] (W) at (4.1, 2.) {$W$};

\draw[->] (W) -- (AT);
\draw[->] (W) -- (AT1);
\draw[->] (W) -- (A1);
\draw[->] (W) -- (A0);

\draw[->] (AT) -- (AT1);
\draw[->] (A1) -- (A0);

\draw[->] (AT) to[bend left=20] (A1);
\draw[->] (AT) to[bend left=25] (A0);

\draw[<->, dashed] (A1) to[bend left=30] (AT1);
\draw[<->, dashed] (A0) to[bend left=40] (AT1);
\draw[<->, dashed] (A0) to[bend left=30] (A1);

\end{tikzpicture}%
}
\caption{Partially projected high-level C-DAG. Solid and dashed
bidirected edges represent directed causal relations and latent
dependence induced by partial projection, respectively.}
\label{fig:partially-projected-cdag}
\end{figure}
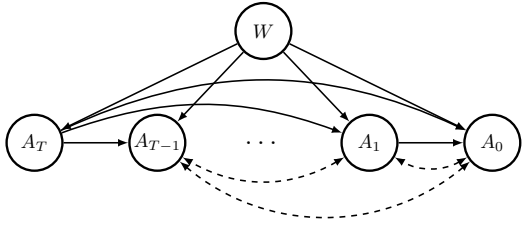

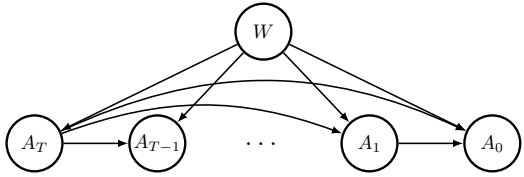
\begin{figure}[t]
\centering
\resizebox{0.82\columnwidth}{!}{%
\begin{tikzpicture}[
    >=latex,
    thick,
    state/.style={circle, draw, very thick, minimum size=1.cm, inner sep=0pt, font=\normalsize}
]

\node[state] (AT) at (0,0) {$A_T$};
\node[state] (AT1) at (2.2,0) {$A_{T-1}$};
\node[font=\Large] (dots) at (4.1,0) {$\dots$};
\node[state] (A1) at (6,0) {$A_1$};
\node[state] (A0) at (8.2,0) {$A_0$};

\node[state] (W) at (4.1, 2.) {$W$};

\draw[->] (W) -- (AT);
\draw[->] (W) -- (AT1);
\draw[->] (W) -- (A1);
\draw[->] (W) -- (A0);

\draw[->] (AT) -- (AT1);
\draw[->] (A1) -- (A0);

\draw[->] (AT) to[bend left=20] (A1);
\draw[->] (AT) to[bend left=25] (A0);

\end{tikzpicture}%
}
\caption{Causal graph induced by the transformer NCM \(\widehat{\mathcal{M}}_H(\theta)\).}
\label{fig:ncm-graph}
\end{figure}

The following results show that if a high-level SCM \(\mathcal{M}_H\) is compatible with \(\mathcal{G}^{\dagger}\) and reproduces the observational high-level distribution, then it recovers any identifiable causal query \(\tau(Q_{w,c})\) that is identifiable in \(\mathcal{G}^{\dagger}\). Consequently, $\mathcal{M}_H$ is $Q_{w,c}$--$\tau$ consistent with
$\mathcal{M}_L$.

\begin{proposition}[{\cite[Theorem~2]{xia2025causal}}]
\label{prop:cdag-suffi-nece}
Let $\mathcal{M}_L$ be a low-level SCM over variables $\mathcal{V}_L$, and let
\(
\tau:\operatorname{Dom}(\mathcal{V}_L)
\to
\operatorname{Dom}(\mathcal{V}_H)
\)
be a constructive abstraction function.
Let $\mathcal{V}_H^\dagger$ denote the set of AIC-violating high-level variables. Then the partially projected C-DAG $\mathcal{G}^{\dagger}$ induced by $\mathcal{M}_L$ with respect to
the abstraction and $\mathcal{V}_H^\dagger$ completely describes
the causal constraints over $\mathcal{V}_H$.
\end{proposition}

\begin{proposition}[Identifiability of the High-Level Audit Query]
\label{prop:high-level-identifiability}
Under the partially projected C-DAG
$\mathcal{G}^{\dagger}$ defined above, for every $a_T$ and $w$
in the support of the observational high-level distribution, the
high-level causal query $\tau(Q_{w,c})$ in
Eq.~\eqref{eq:high-level-query},
\(
P^{\mathcal{M}_H}
\bigl(a_0 \mid a_T,c,\doop(w)\bigr),
\)
is identifiable from the observational high-level distribution.
\end{proposition}

\begin{corollary}[$Q$--$\tau$ Consistency of the High-Level SCM]
\label{cor:audit-consistency}
Let $\mathcal{G}^{\dagger}$ be the partially projected C-DAG
induced by the low-level SCM $\mathcal{M}_L$ and the constructive
abstraction function $\tau$. Suppose that the high-level SCM
$\mathcal{M}_H$ is compatible with $\mathcal{G}^{\dagger}$ and
reproduces the observational high-level distribution induced by
$\mathcal{M}_L$ under $\tau$. Then, for every $w$ and $c$ for
which Proposition~\ref{prop:high-level-identifiability} applies,
$\mathcal{M}_H$ is $Q_{w,c}$--$\tau$ consistent with
$\mathcal{M}_L$, that is,
\(
(Q_{w,c})^{\mathcal{M}_L}
=
\tau(Q_{w,c})^{\mathcal{M}_H}.
\)
\end{corollary}

\subsection{Transformer Neural Causal Model}
\label{sec:transformer-ncm}

Next, we develop a transformer neural causal model (NCM) \cite{xia2024neural} as an amortized
predictor of the high-level query $\tau(Q_{w,c})$, assuming
that the component mappings
$\tau_W,\tau_T,\ldots,\tau_0$ are given and fixed. Applying these
mappings to a low-level diffusion trajectory produces the high-level
trajectory
\(
    \bigl(
        w,\tau_T(x_T),\tau_{T-1}(x_{T-1}),
        \ldots,\tau_0(x_0)
    \bigr),
\)
which is used for training.

\paragraph{Transformer neural causal model.}
We represent the high-level SCM using a neural causal model (NCM),
\[
\widehat{\mathcal{M}}_H(\theta)
=
\left\langle
\widehat{\mathcal{U}},
\mathcal{V}_H,
\widehat{\mathcal{F}}_\theta,
P(\widehat{\mathcal{U}})
\right\rangle,
\]
where \(\mathcal{V}_H=\{W,A_T,\ldots,A_0\}\) and
\(\widehat{\mathcal{F}}_\theta\) consists of probabilistic neural
structural mechanisms jointly parameterized by a transformer. The
high-level states are ordered according to the denoising process as
\((A_T,A_{T-1},\ldots,A_0)\). When predicting \(A_t\), the causal
attention mask allows the transformer to access the complete causal
prefix \(A_{t+1:T}=(A_{t+1},\ldots,A_T)\), together with the fixed
prompt \(c\) and CFG scale \(w\). The transformer therefore
parameterizes
\(
P_\theta(a_t\mid a_{t+1:T},c,w)
\)
for every \(t\in\{0,\ldots,T-1\}\).
Given \(N\) mapped high-level trajectories, the NCM is trained by
minimizing the negative log-likelihood
\[
\widehat{\mathcal{L}}_N(\theta)
=
-\frac{1}{N}
\sum_{i=1}^{N}
\sum_{t=0}^{T-1}
\log
p_\theta
\left(
a_t^{(i)}
\mid
a_{t+1:T}^{(i)},c,w^{(i)}
\right).
\]

It is worth noting that the causal graph induced by \(\widehat{\mathcal{M}}_H(\theta)\),
denoted by \(\mathcal{G}_\theta\) and illustrated in Figure~\ref{fig:ncm-graph}, differs from the partially projected C-DAG
\(\mathcal{G}^{\dagger}\) in Figure~\ref{fig:partially-projected-cdag}. Specifically, \(\mathcal{G}_\theta\)
contains directed edges from the complete causal prefix but contains
no bidirected edges. Consequently,
\(\widehat{\mathcal{M}}_H(\theta)\) is not generally causally
equivalent to a model compatible with \(\mathcal{G}^{\dagger}\), and
the two models may disagree on interventions involving an intermediate
state \(A_t\).
Nevertheless, Proposition~\ref{prop:high-level-identifiability}
establishes that the target query \(\tau(Q_{w,c})\) is identifiable
from the high-level observational distribution. Moreover, the complete
directed ordering imposes no additional conditional-independence
restrictions among the high-level diffusion states. It follows that
the NCM recovers the target causal query whenever its learned joint
distribution agrees with the high-level observational distribution, as
formalized in the following result.

\begin{proposition}
\label{prop:likelihood-query-approximation}
Given the high-level SCM \(\mathcal{M}_H\) and the transformer NCM
\(\widehat{\mathcal{M}}_H(\theta)\) defined above, define the excess negative log-likelihood as
\[
\begin{aligned}
\varepsilon_\theta
\defeq
\mathbb{E}_{A_T,W}
\Big[
D_{\mathrm{KL}}\big(
&P^{\mathcal{M}_H}
(A_{T-1:0}\mid A_T,c,W)\\
&\big\|\,
P_\theta
(A_{T-1:0}\mid A_T,c,W)
\big)
\Big].
\end{aligned}
\]
Then,
\[
\begin{aligned}
\mathbb{E}_{A_T,W}
\Bigg[
\sup_{a_0\in[0,1]}
\Big|
&P_\theta
(a_0\mid A_T,c,\doop(W))\\
&-
P^{\mathcal{M}_H}
(a_0\mid A_T,c,\doop(W))
\Big|
\Bigg]
\leq
\sqrt{\frac{\varepsilon_\theta}{2}},
\end{aligned}
\]
that is, the causal query computed by
\(\widehat{\mathcal{M}}_H(\theta_n)\) converges to the corresponding
query in \(\mathcal{M}_H\).
\end{proposition}

\begin{table*}[t]
    \small
    \centering
    \caption{Wasserstein-1 distance ($D$) and normalized discrepancy
    ratio ($R$) between low- and high-level final-attribute
    distributions.}
    \label{tab:fidelity-wasserstein-results}

    \begin{tabular}{ccccccc}
        \toprule
        & \multicolumn{2}{c}{Lawyer}
        & \multicolumn{2}{c}{Librarian}
        & \multicolumn{2}{c}{Scientist} \\
        \cmidrule(lr){2-3}
        \cmidrule(lr){4-5}
        \cmidrule(lr){6-7}

        CFG $w$
        & $D$ & $R$
        & $D$ & $R$
        & $D$ & $R$ \\
        \midrule

        2
        & $0.056 \pm 0.029$ & 0.804962
        & $0.070 \pm 0.022$ & 1.291038
        & $0.044 \pm 0.023$ & 0.444664 \\

        4
        & $0.031 \pm 0.020$ & 0.311368
        & $0.028 \pm 0.007$ & 0.923073
        & $0.058 \pm 0.027$ & 0.642809 \\

        6
        & $0.027 \pm 0.019$ & 0.221672
        & $0.032 \pm 0.010$ & 1.087213
        & $0.079 \pm 0.031$ & 1.081464 \\

        8
        & $0.024 \pm 0.017$ & 0.128558
        & $0.019 \pm 0.005$ & 0.829404
        & $0.070 \pm 0.031$ & 0.957128 \\

        10
        & $0.024 \pm 0.017$ & 0.112873
        & $0.011 \pm 0.003$ & 0.834020
        & $0.050 \pm 0.023$ & 0.573595 \\

        12
        & $0.033 \pm 0.022$ & 0.458389
        & $0.014 \pm 0.004$ & 1.178870
        & $0.063 \pm 0.023$ & 0.872928 \\

        15
        & $0.026 \pm 0.019$ & 0.180890
        & $0.014 \pm 0.004$ & 0.812535
        & $0.071 \pm 0.022$ & 1.058052 \\

        20
        & $0.027 \pm 0.020$ & 0.188876
        & $0.010 \pm 0.004$ & 0.496726
        & $0.093 \pm 0.018$ & 1.558177 \\

        \bottomrule
    \end{tabular}
\end{table*}

\section{Evaluation of Causal Abstraction Fidelity}
\label{sec:causal-abstraction-fidelity}


\subsection{Implementation Details}
\label{subsec:fidelity-implementation}

The low-level model is Stable Diffusion v1.5 with $T=50$ denoising
steps.  For each intermediate latent $x_t$, we decode the corresponding
image and use CLIP~\cite{radford2021learning} with paired masculine and
feminine text prompts to obtain the scalar abstraction
$a_t=\tau_t(x_t)\in[0,1]$.

The high-level model is a sliding-window causal transformer whose input
at timestep $t$ consists of the normalized attribute score, normalized
scheduler timestep, and CFG scale.  Given a causal history
$\mathcal{H}^{(L)}_t$ of at most $L$ states, the transformer models the
scaled attribute residual as
\begin{equation}
\begin{gathered}
    y_t=10\left(\widetilde a_{t-1}-\widetilde a_t\right),
    \qquad (\mu_t,r_t)=f_\theta(\mathcal{H}^{(L)}_t,w),\\
    y_t\mid\mathcal{H}^{(L)}_t,w
    \sim\mathcal{N}\!\left(
        \mu_t,\operatorname{softplus}(r_t)+10^{-6}
    \right).
\end{gathered}
\label{eq:fidelity-transformer}
\end{equation}
The model is trained using Gaussian negative log-likelihood and rolled
out autoregressively from $A_T$ to $A_0$.  Its architecture uses
5 transformer layers, 8 attention heads,
and hidden dimension 512. The detailed architecture is provided in the appendix.

\subsection{Experimental Setup}
\label{subsec:fidelity-setup}

We consider three gender-neutral prompts: \textit{``a portrait photo of
a lawyer''}; \textit{``a photo of a person working as a friendly
librarian at a cozy library, surrounded by bookshelves filled with
various genres''}; and \textit{``a photo of a person working as a lead
scientist in a cutting-edge laboratory, wearing a lab coat and
goggles.''} 
For each prompt and training CFG scale
$w\in\{0,1,3,5,7,9\}$, we generate trajectories for 1,000 seeds,
divided into 800 training and 200 held-out seeds. For evaluation,
we additionally generate low-level trajectories for the 200
held-out seeds at the unseen CFG scales
$w\in\{2,4,6,8,10,12,15,20\}$.
We use $w_{\mathrm{ref}}=8$ as the default reference scale and, unless otherwise specified, a transformer window size of $L=25$. For prompt $c$ and CFG scale $w$, fidelity is measured using the
Wasserstein-1 distance between the empirical low- and high-level
final-attribute distributions,
\begin{equation}
    D_{W_1}(w,c)
    =W_1\!\left(\widehat P_L^{w,c},\widehat P_H^{w,c}\right).
    \label{eq:fidelity-wasserstein}
\end{equation}
We estimate sampling variance using $B=1{,}000$ bootstrap resamples of
the 200 held-out seeds.  To calibrate finite-sample variation, we also
draw two independent bootstrap samples from the low-level scores and
compute their Wasserstein distance $D_{\mathrm{LL}}^{(b)}(w,c)$.
The normalized discrepancy ratio is
\begin{equation}
    R_{W_1}(w,c)
    =\frac{D_{W_1}(w,c)}
    {\operatorname{Quantile}_{0.95}
      \left(\{D_{\mathrm{LL}}^{(b)}(w,c)\}_{b=1}^{B}\right)}.
    \label{eq:fidelity-discrepancy-ratio}
\end{equation}
Thus, $R_{W_1}\leq1$ means the low-high discrepancy is no larger than the 95th-percentile low-low resampling reference.

{\bf\noindent Selection of window size.}
We evaluate candidate window sizes
$L\in\{1,5,10,15,20,25,30,35,40,45,50\}$ at $w_{\mathrm{ref}}$
using the Wasserstein distance on held-out validation trajectories. Based on the average performance across the three prompts,
we select $L=25$ for all subsequent experiments. Detailed
figures are provided in the appendix.

{\bf\noindent Fidelity of the final-attribute distribution.}
Table~\ref{tab:fidelity-wasserstein-results} reports distributional fidelity at the unseen CFG scales using $L=25$, with uncertainty estimated by bootstrap resampling of the held-out seeds. Across the 24 prompt--CFG configurations, the Wasserstein distance ranges from $0.010$ to $0.093$, and 18 configurations (75\%) satisfy $R\leq1$ (i.e., the low--high discrepancy does not exceed the bootstrap reference). 
Fidelity varies across prompts and scales.
 
Nevertheless, the overall results demonstrate that the high-level model generally preserves the final-attribute distribution at CFG scales not observed during training.

\begin{figure*}[t]
    \centering
    \begin{subfigure}[t]{0.32\textwidth}
        \centering
        \includegraphics[width=1\linewidth]
    {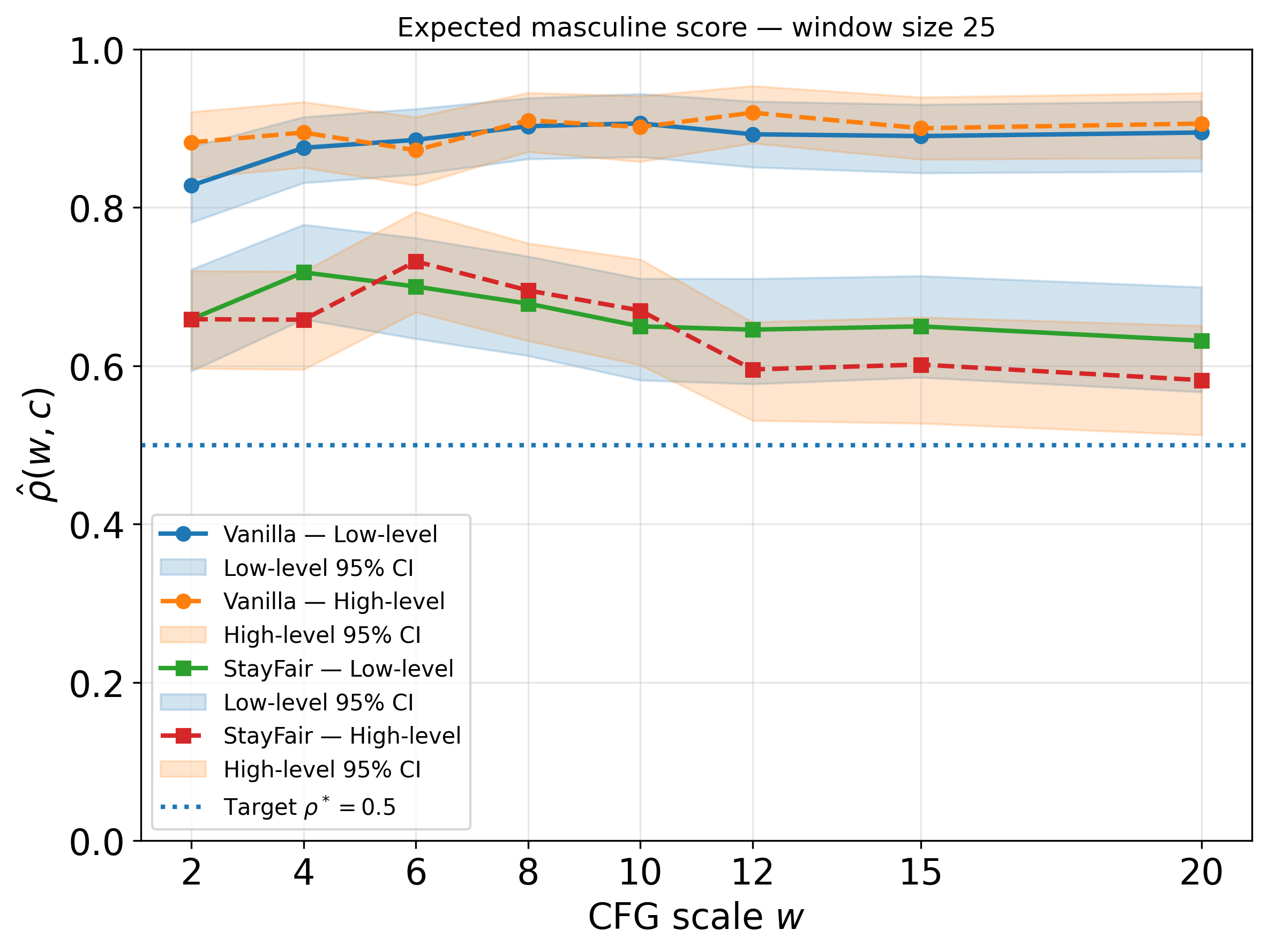}
        \caption{Lawyer}
    \end{subfigure}
    \hfill
    \begin{subfigure}[t]{0.32\textwidth}
        \centering
        \includegraphics[width=1\linewidth]
    {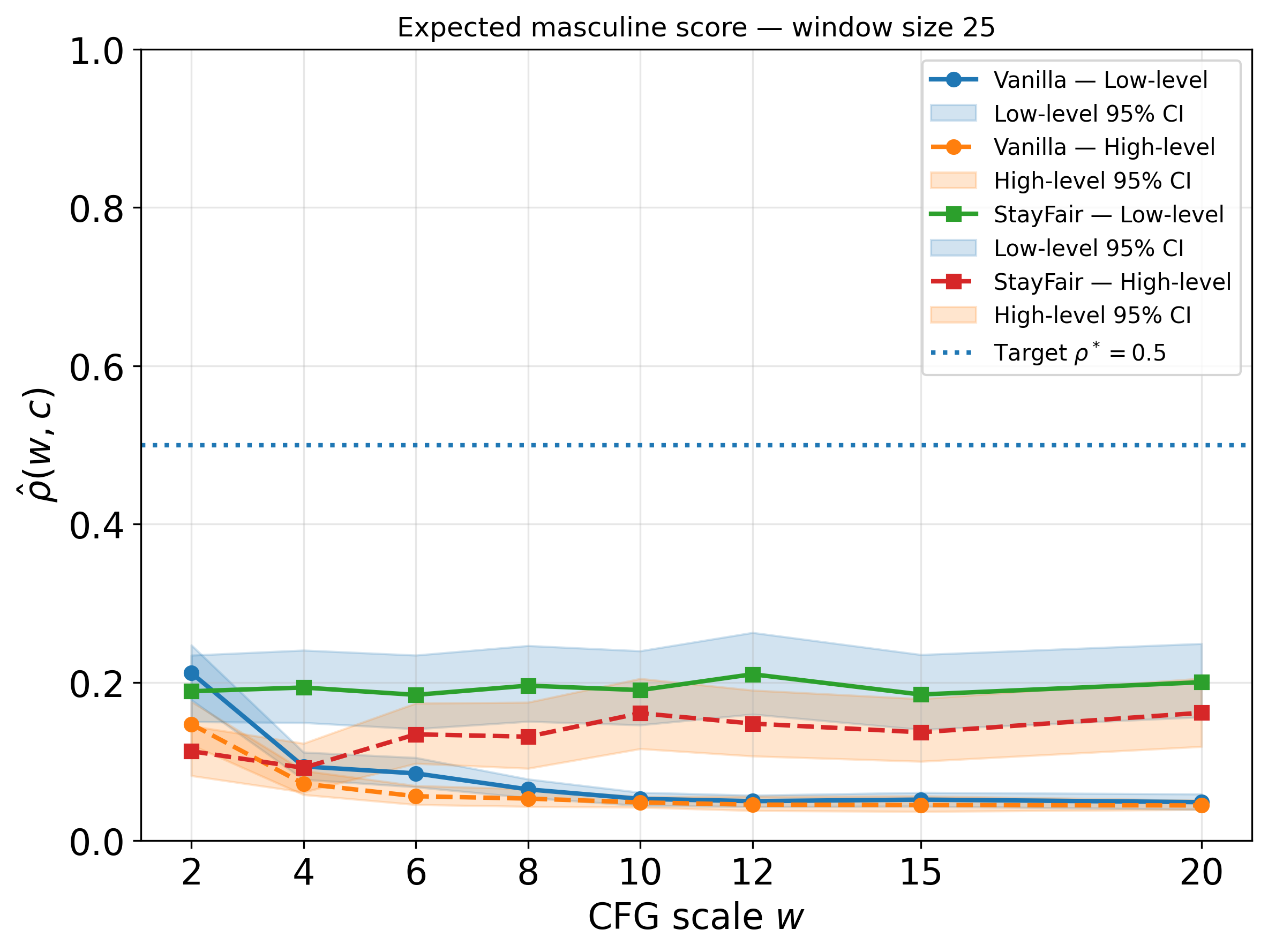}
        \caption{Librarian}
    \end{subfigure}
    \hfill
    \begin{subfigure}[t]{0.32\textwidth}
        \centering
        \includegraphics[width=1\linewidth]
    {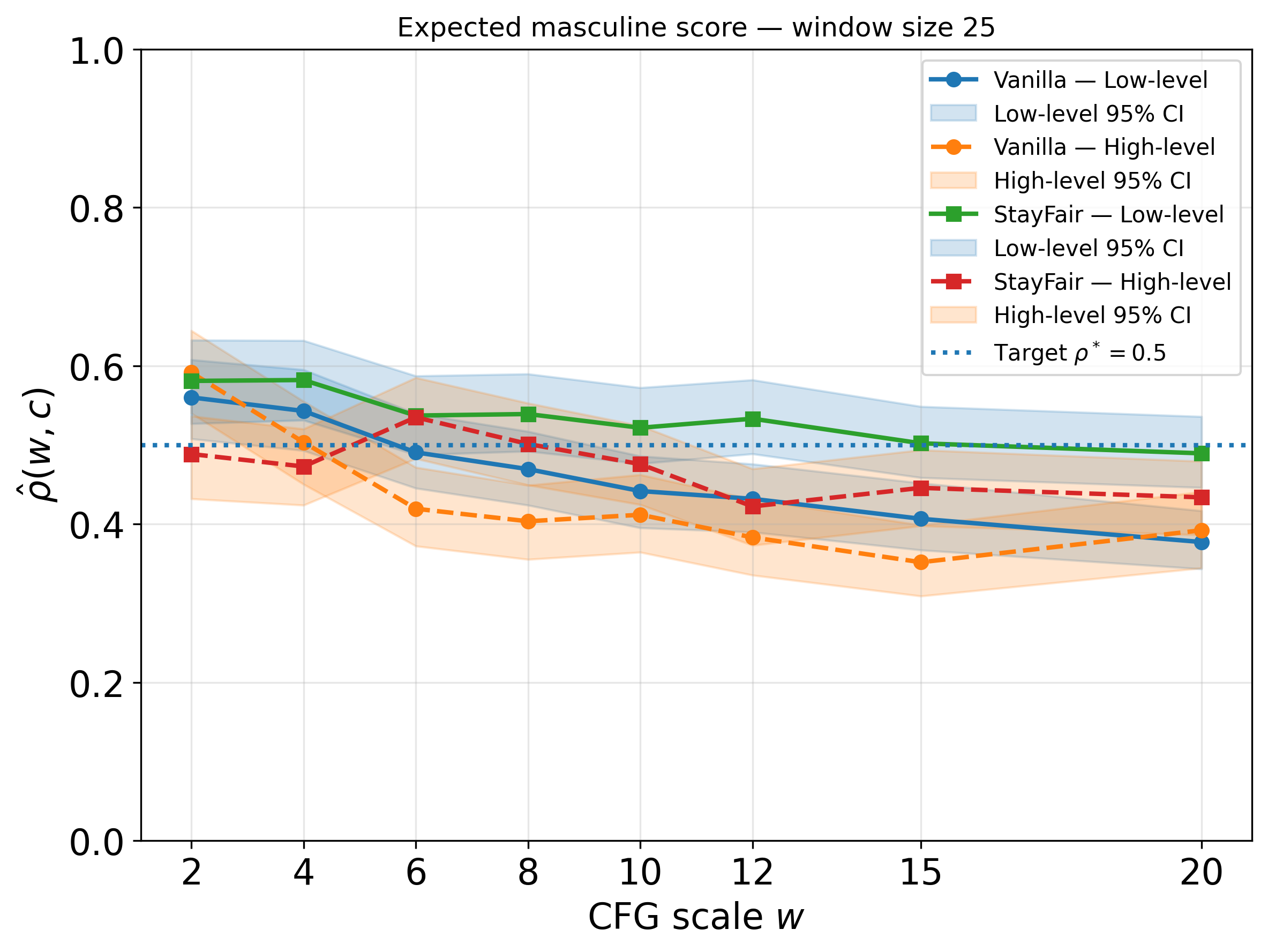}
        \caption{Scientist}
    \end{subfigure}
    \caption{Fairness audits across classifier-free guidance scales for
    standard Stable Diffusion v1.5 and SD1.5+StayFair.  For each model,
    solid and dashed curves show the low-level and abstract-model
    estimates, respectively.  The horizontal line at $0.5$ denotes the
    demographic-parity reference.}
    \label{fig:fairness-demo-guidance-curves}
\end{figure*}

{\bf\noindent Empirical assessment of AIC violation.}
We empirically assess whether the feature mapping preserves
the low-level information relevant to downstream attribute evolution.
At $t=35$, among matched held-out trajectory pairs $i,j$ satisfying
$\lvert a_{35}^{(i)}-a_{35}^{(j)}\rvert<0.005$, the distributions of
both $\lvert a_{34}^{(i)}-a_{34}^{(j)}\rvert$ and
$\lvert a_{0}^{(i)}-a_{0}^{(j)}\rvert$ remain widely dispersed, even as
$\lvert a_{35}^{(i)}-a_{35}^{(j)}\rvert$ approaches zero. This  
provides empirical evidence that the abstraction discards causally relevant information and violates the AIC. Detailed procedures and figures are provided in the appendix material.

\section{Fairness Audit Demonstrations}
\label{sec:fairness-audit-demonstrations}

We demonstrate the proposed auditing instrument on two configurations:
standard Stable Diffusion v1.5 and Stable Diffusion v1.5 equipped with
StayFair~\cite{kim2026stayfair}.  StayFair modifies classifier-free
guidance through a prompt-dependent null-embedding adjustment, allowing
us to compare standard and fairness-enhanced guidance while keeping the
underlying diffusion model fixed.
Audit conclusions are determined using the confidence interval and tolerance $\delta=0.10$.

{\bf\noindent Decision fidelity.}  
Figure~\ref{fig:fairness-demo-guidance-curves} compares the expected masculine scores estimated by the low- and high-level models. Overall, the high-level estimates reproduce the prompt- and scale-dependent trends of the low-level outputs. For the lawyer prompt, standard SD1.5 produces consistently high masculine scores, whereas
StayFair moves them toward the parity reference. For the librarian prompt, standard SD1.5 produces consistently low scores, and StayFair provides only a modest improvement. For the scientist prompt, both standard SD1.5 and StayFair produce scores near parity. Despite some numerical discrepancies, the abstract model captures these overall patterns and the fairness profiles underlying the audit.

{\bf\noindent Computational savings.}
We compare the per-sample inference time of the low- and high-level models. For the low-level diffusion model, runtime is measured from the initial state $x_T$
through the complete denoising process, VAE decoding, and CLIP scoring
of the final image to obtain $a_0$. Intermediate CLIP evaluations are
excluded because they are unnecessary when computing only the final
attribute. Averaged over 500 seeds, the low-level runtime is
$5.035$ seconds per sample.
For the high-level model, runtime is measured from the initial attribute
$a_T$ through generation of the complete high-level trajectory to
$a_0$. The average runtime is $0.259$ seconds per sample, corresponding
to a $19.4\times$ inference speedup, or a 94.9\% reduction in runtime,
relative to the low-level model.

{\bf\noindent Audit conclusions.}
For each prompt--CFG configuration, we generate $1{,}000$ samples using only the high-level audit instrument. We use these samples to estimate $\widehat{\Phi}_K$, construct its confidence interval, and determine
the audit outcome under $\alpha=0.10$ and $\delta=0.1$.
Table~\ref{tab:fairness-demo-summary} summarizes the results.

\begin{table}\small
    \centering
    \caption{Auditing results for standard SD1.5 and SD1.5+StayFair (supported, violated, or inconclusive).}
    \label{tab:fairness-demo-summary}
    \begin{tabular}{ccccc}
        \toprule
        & \multicolumn{2}{c}{Standard SD1.5}
        & \multicolumn{2}{c}{SD1.5 + StayFair} \\
        \cmidrule(lr){2-3}\cmidrule(lr){4-5}
        Prompt & $\widehat\Phi_K$ & Audit
               & $\widehat\Phi_K$ & Audit \\
        \midrule
        Lawyer
            & $[0.367,0.468]$ & Vio.
            & $[0.162,0.263]$ & Vio. \\
        Librarian
            & $[0.405,0.506]$ & Vio.
            & $[0.338,0.439]$ & Vio. \\
        Scientist
            & $[0.068,0.168]$ & Incon.
            & $[0.041,0.142]$ & Incon. \\
        \bottomrule
    \end{tabular}
\end{table}

\section{Conclusions}
We presented a query-specific causal-abstraction instrument for
efficient fairness auditing of text-to-image diffusion models. Our
analysis shows that an identifiable audit query can be recovered despite
lossy abstraction and differences between the learned and projected
causal graphs.
Across unseen guidance scales, the instrument generally preserved
final-attribute distributions and reproduced prompt- and scale-dependent
fairness trends, while achieving a significant reduction in per-sample inference time. The results also demonstrate the importance of auditing guidance settings and prompt formulations. Future work will extend the framework to broader prompt
distributions, multiple attributes, alternative high-level SCM implementations,
and additional generative-model families.

{
    \small
    \bibliographystyle{ieeenat_fullname}
    \bibliography{main}
}

\appendix
\section{Proof of Proposition 1}

\begin{proof}
Because \(A_T=\tau_T(X_T)\), the law of total probability gives
\[
\begin{split}
&P^{\mathcal{M}_L}
\left(
    \tau_0(X_0)\leq a_0
    \mid c,\doop(w)
\right)\\
&=
\mathbb{E}_{A_T}
\left[
    \mathbb{E}_{X_T\mid\tau_T(X_T)=A_T}
    \left[
        Q_{w,c}(a_0\mid X_T)
    \right]
\right].
\end{split}
\]
By \(Q_{w,c}\)-\(\tau\) consistency, the inner expectation equals the
corresponding high-level query. Therefore,
\[
\begin{split}
&P^{\mathcal{M}_L}
\left(
    \tau_0(X_0)\leq a_0
    \mid c,\doop(w)
\right)\\
&=
P^{\mathcal{M}_H}
\left(
    A_0\leq a_0
    \mid c,\doop(w)
\right).
\end{split}
\]
Equality of the target-feature distributions implies equality of their
expectations, so the two models induce the same value of
\(\rho(w,c)\).
\end{proof}

\section{Proof of Proposition 2}

\begin{proof}
Define the intervariable clustering
\[
    \mathcal{C}
    =
    \{\mathcal{C}_W,\mathcal{C}_T,\ldots,\mathcal{C}_0\},
    \qquad
    \mathcal{C}_W=\{W\},
    \quad
    \mathcal{C}_t=\{X_t\}.
\]
These clusters partition \(\mathcal{V}_L\). Moreover, there is a
bijection between the intervariable clusters and the high-level
variables: \(\mathcal{C}_W\) corresponds to \(W\), while
\(\mathcal{C}_t\) corresponds to \(A_t\).

For each \(\mathcal{C}_t\), define the intravariable clustering
\[
    \mathfrak{D}_t
    =
    \left\{
        \tau_t^{-1}(a_t):
        a_t\in\mathcal{A}_t
    \right\}.
\]
Because \(\mathcal{A}_t=\tau_t(\mathcal{X})\), the fibers of
\(\tau_t\) are nonempty, mutually disjoint, and cover
\(\mathcal{X}\). Therefore, \(\mathfrak{D}_t\) partitions the domain
of \(X_t\), and each fiber \(\tau_t^{-1}(a_t)\) corresponds to exactly
one high-level value \(a_t\). For \(W\), the intravariable clustering
consists of the singleton sets
\(\mathfrak{D}_W=\{\{w\}:w\in\mathcal{W}\}\).

Finally, \(\tau\) decomposes across the intervariable clusters as
\[
    \tau(w,x_T,\ldots,x_0)
    =
    \bigl(
        \tau_W(w),
        \tau_T(x_T),
        \ldots,
        \tau_0(x_0)
    \bigr).
\]
Thus, the intervariable clusters correspond bijectively to the
high-level variables, their intravariable clusters correspond
bijectively to the high-level values, and \(\tau\) acts separately on
each cluster. Hence, \(\tau\) is a constructive abstraction function.
\end{proof}

\section{Proof of Proposition 4}
\begin{proof}
Because \(W\) has neither parents nor incident bidirected edges, no
backdoor path from \(W\) to \(A_0\) exists, including after
conditioning on the pre-intervention variable \(A_T\). Therefore, the
empty set satisfies the backdoor criterion, and the intervention
\(\doop(w)\) can be replaced by conditioning on \(W=w\).
\end{proof}

\section{Proof of Corollary 5}
\begin{proof}
Let $\mathcal{Z}$ denote the observational distribution induced by
$\mathcal{M}_L$. By assumption, $\mathcal{M}_H$ is
$\mathcal{Z}$--$\tau$ consistent with $\mathcal{M}_L$.
By Proposition 4,
$\tau(Q_{w,c})$ is identifiable from
$\mathcal{G}^{\dagger}$ and the corresponding observational
high-level distribution $\tau(\mathcal{Z})$.
By the dual abstract identification result of
\cite{xia2025causal}, $Q_{w,c}$ is therefore $\tau$-identifiable
from $\mathcal{G}^{\dagger}$ and $\mathcal{Z}$.
Consequently, observational $\tau$-consistency implies
$Q_{w,c}$--$\tau$ consistency, so
\[
(Q_{w,c})_{\mathcal{M}_L}
=
\tau(Q_{w,c})_{\mathcal{M}_H}.
\]
\end{proof}

\section{Proof of Proposition 6}

\begin{proof}
By Proposition 4, the target
causal query in \(\mathcal{M}_H\) satisfies
\[
P^{\mathcal{M}_H}(a_0\mid a_T,c,\doop(w))
=
P^{\mathcal{M}_H}(a_0\mid a_T,c,w).
\]
The transformer NCM has the same property because \(W\) is an
intervenable root variable and its complete-prefix factorization
directly parameterizes the conditional distribution given \(w\).
Therefore,
\[
P_{\theta}(a_0\mid a_T,c,\doop(w))
=
P_{\theta}(a_0\mid a_T,c,w).
\]

For fixed \(a_T\) and \(w\), define
\[
\begin{aligned}
d_\theta(a_T,w)
&:=
D_{\mathrm{KL}}\Big(
P^{\mathcal{M}_H}
    (A_{T-1:0}\mid a_T,c,w) \\
&\qquad \,\big\|\,
P_\theta
    (A_{T-1:0}\mid a_T,c,w)
\Big).
\end{aligned}
\]
Marginalizing \(A_{T-1:0}\) to \(A_0\) is a measurable
transformation. Hence, by the data-processing inequality for
KL divergence,
\[
\begin{aligned}
&D_{\mathrm{KL}}\Big(
P^{\mathcal{M}_H}(A_0\mid a_T,c,w)
\,\big\|\,
P_\theta(A_0\mid a_T,c,w)
\Big)\\
&\hspace{35mm}
\leq d_\theta(a_T,w).
\end{aligned}
\]
Pinsker's inequality then gives
\[
\begin{aligned}
&\left\|
P^{\mathcal{M}_H}(A_0\mid a_T,c,w)
-
P_\theta(A_0\mid a_T,c,w)
\right\|_{\mathrm{TV}}\\
&\hspace{35mm}
\leq
\sqrt{\frac{d_\theta(a_T,w)}{2}}.
\end{aligned}
\]

For every \(a_0\in[0,1]\), the event
\(\{A_0\leq a_0\}\) is measurable. Because the difference in the
probabilities of any measurable event is bounded by the total
variation distance, we obtain
\[
\begin{aligned}
\sup_{a_0\in[0,1]}
\Big|
&P_\theta(a_0\mid a_T,c,\doop(w))
-
P^{\mathcal{M}_H}(a_0\mid a_T,c,\doop(w))
\Big|\\
&\leq
\sqrt{\frac{d_\theta(a_T,w)}{2}}.
\end{aligned}
\]
Taking the expectation over \((A_T,W)\) and applying Jensen's
inequality yields
\[
\begin{aligned}
&\mathbb{E}_{A_T,W}\Bigg[
\sup_{a_0\in[0,1]}\big|
 P_\theta(a_0\mid A_T,c,\doop(W)) \\
& \qquad {}-P^{\mathcal{M}_H}(a_0\mid A_T,c,\doop(W))
\big|\Bigg] \\
&\leq
\mathbb{E}_{A_T,W}
\left[
\sqrt{\frac{d_\theta(A_T,W)}{2}}
\right]\\
&\leq
\sqrt{
\frac{
\mathbb{E}_{A_T,W}[d_\theta(A_T,W)]
}{2}
}
=
\sqrt{\frac{\varepsilon_\theta}{2}}.
\end{aligned}
\]
Thus, as \(\varepsilon_\theta\) approaches zero, the causal query
computed by \(\widehat{\mathcal{M}}_H(\theta)\) converges to the
high-level causal query.
\end{proof}

\section{Implementation Details}
The high-level reverse-dynamics model was implemented as a causal Transformer that models the evolution of the gender attribute over 50 reverse-diffusion steps. Each input token contained the current gender score, diffusion timestep, and classifier-free guidance (CFG) value. These variables were embedded separately and combined with a learned positional embedding. The model consisted of five Transformer blocks with a hidden dimension of 512, eight attention heads, a feed-forward dimension of 1,024, GELU activations, and a dropout rate of 0.1. A causal sliding-window attention mask allowed each token to attend only to itself and a specified number of preceding tokens. The output layer predicted the mean and variance of a Gaussian distribution over the change in the normalized gender score between consecutive diffusion steps.

The data were divided into training and testing subsets using an 80/20 seed-based split, ensuring that trajectories generated from the same initial noise seed did not appear in both subsets. The gender scores and diffusion timesteps were normalized using statistics calculated only from the training data, whereas the CFG values were provided without normalization. The models were trained on CFG values \(w\in\{0,1,3,5,7,9\}\) with a batch size of 32. Training used the AdamW optimizer with a learning rate of \(2\times10^{-5}\), weight decay of \(10^{-5}\), Gaussian negative log-likelihood loss, and gradient clipping at 1.0. Separate models were trained for each attention-window size using the same data split, normalization statistics, initialization seed, and data-shuffling seed.

Testing was performed autoregressively on held-out seeds and unseen CFG values \(w\in\{2,4,6,8,10,12,15,20\}\). Each rollout was initialized with the low-level model’s gender score at the initial reverse-diffusion state. At each subsequent step, the Transformer used the previously generated high-level states to predict and sample the next residual, which was added to the current state to construct the complete trajectory. The generated final score was transformed back to a masculine probability and compared with the corresponding low-level Stable Diffusion result. Agreement between the two levels was evaluated using final-score errors, differences in the expected masculine score and fairness gap, and the Wasserstein-1 distance between their final-score distributions.

\section{Selection of Window Size}
\label{sec:supp-window-size}

The transformer window size \(L\) determines the number of preceding
abstract diffusion states available when predicting the next state. A
larger window can capture longer-range dependencies induced by the
lossy abstraction, but also increases the computational cost and
complexity of the model. We compare
\(L\in\{1,5,10,15,20,25,30,35,40,45,50\}\) at the reference CFG scale
\(w_{\mathrm{ref}}=8\). For each prompt and window size, fidelity is
measured by the Wasserstein-1 distance between the low- and high-level
final-attribute distributions. The curves and shaded regions in
Figure 4 show the mean distance and one
standard deviation, respectively.

The results do not exhibit a monotonic relationship between window
size and distributional fidelity. Nevertheless, a rough pattern is
visible: discrepancies fluctuate more across the smaller and
intermediate windows, whereas performance becomes comparatively stable
for \(L\geq25\). In particular, \(L=25\) achieves near-minimal
discrepancy for the lawyer prompt and remains within the stable range
observed for the librarian and scientist prompts. We therefore select
\(L=25\) as a cross-prompt compromise between distributional fidelity,
context length, and computational cost. This selection should be
interpreted as a practical model-selection choice.

\begin{figure*}[t]
    \centering
    \begin{minipage}[t]{0.32\textwidth}
        \centering
        \includegraphics[width=\linewidth]
        {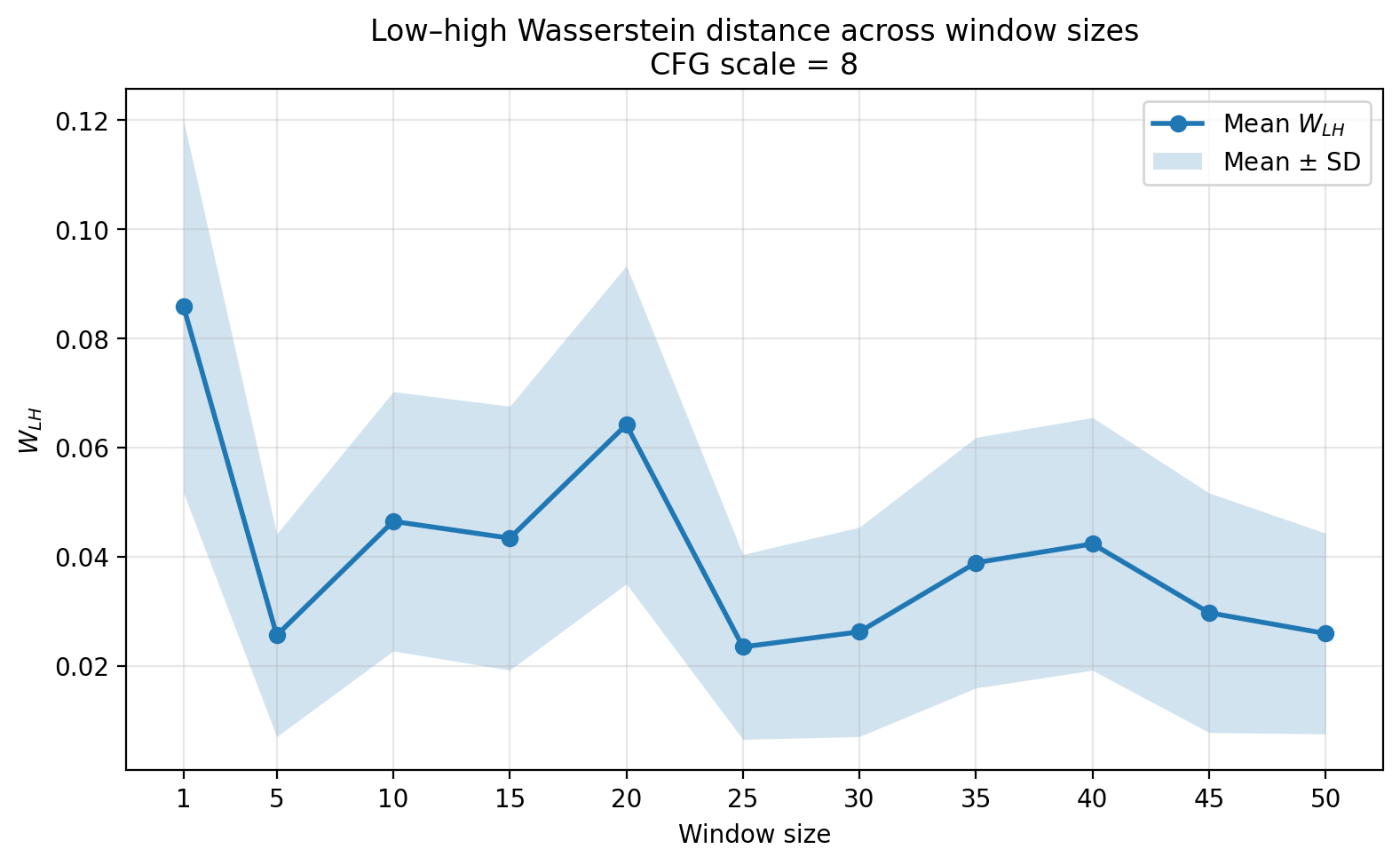}\\
        \textbf{(a) Lawyer}
    \end{minipage}
    \hfill
    \begin{minipage}[t]{0.32\textwidth}
        \centering
        \includegraphics[width=\linewidth]
        {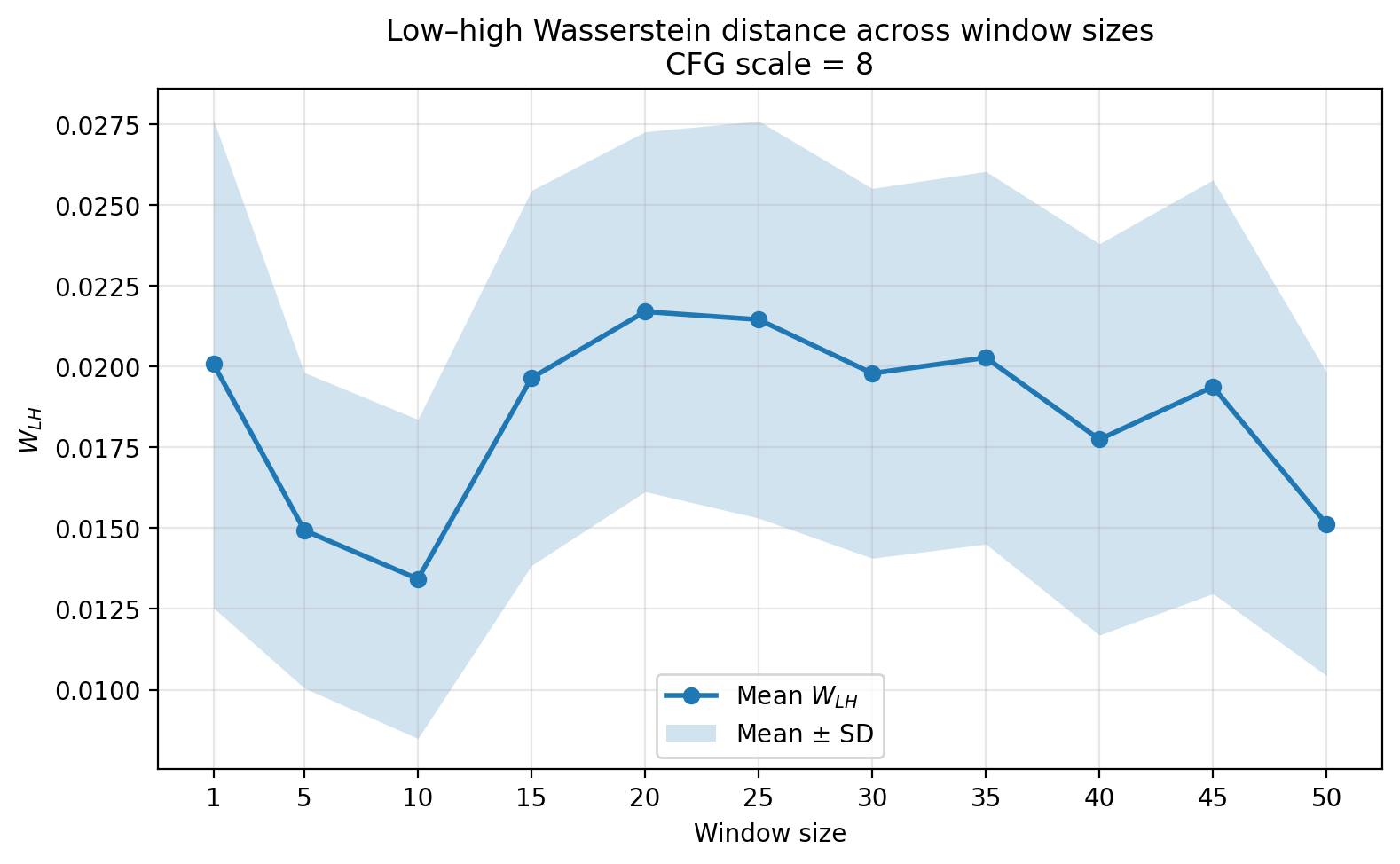}\\
        \textbf{(b) Librarian}
    \end{minipage}
    \hfill
    \begin{minipage}[t]{0.32\textwidth}
        \centering
        \includegraphics[width=\linewidth]
        {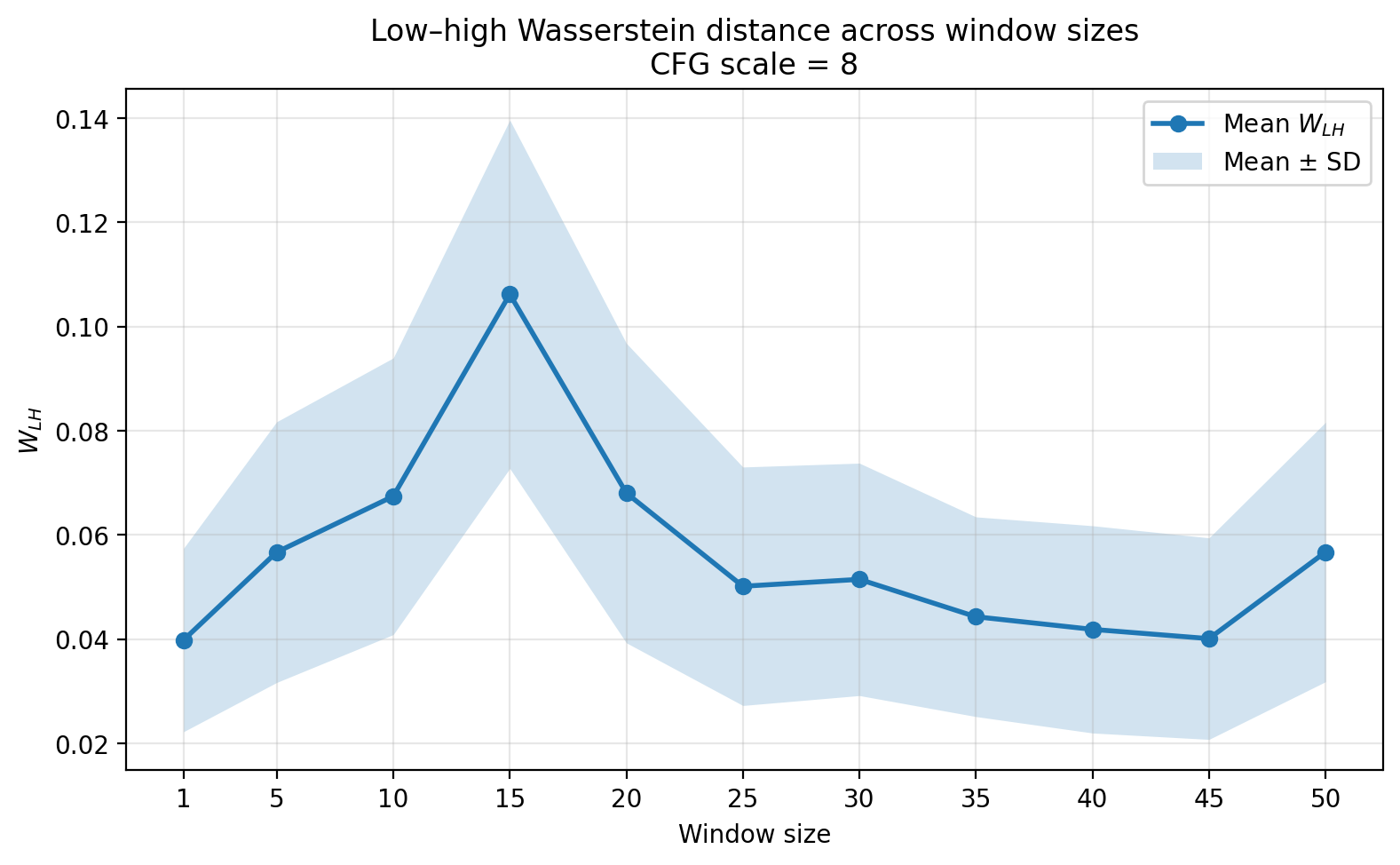}\\
        \textbf{(c) Scientist}
    \end{minipage}
    \caption{Low--high Wasserstein-1 distance across transformer
    window sizes at \(w_{\mathrm{ref}}=8\). The solid curves report
    the mean distance, and the shaded regions represent one standard
    deviation. Lower values indicate greater distributional fidelity.}
    \label{fig:supp-window-size}
\end{figure*}

\section{Empirical Assessment of AIC Violation}
\label{sec:supp-aic-violation}

The Abstract Invariance Condition (AIC) requires low-level states
mapped to the same high-level value to have equivalent causal effects
on downstream high-level variables. In our setting, an AIC violation
can occur when two diffusion states \(x_t^{(i)}\) and \(x_t^{(j)}\)
receive nearly identical attribute values under \(\tau_t\), but the
information discarded by this mapping leads to different subsequent
attribute trajectories.

Because the attribute scores are continuous, exact matches
\(a_t^{(i)}=a_t^{(j)}\) are unlikely in finite samples. We therefore
use approximate matching. At \(t=35\) and CFG scale \(w=8\), we form
pairs of held-out trajectories generated under the same audit
configuration and retain pairs satisfying
\[
    d_{35}^{(i,j)}
    =
    \left|
        a_{35}^{(i)}-a_{35}^{(j)}
    \right|
    <0.005.
\]
For each retained pair, we compare \(d_{35}^{(i,j)}\) with both the
one-step difference
\[
    d_{34}^{(i,j)}
    =
    \left|
        a_{34}^{(i)}-a_{34}^{(j)}
    \right|
\]
and the final-outcome difference
\[
    d_{0}^{(i,j)}
    =
    \left|
        a_{0}^{(i)}-a_{0}^{(j)}
    \right|.
\]

Figure 5 shows that pairs with nearly
identical values at \(t=35\) can diverge substantially at both the next
denoising step and the final output. The one-step differences remain
widely dispersed and can approach \(0.1\), while the final differences
span nearly the entire attribute range. Importantly, this dispersion
does not vanish as \(d_{35}^{(i,j)}\) approaches zero. Thus, proximity
in the abstract state \(A_{35}\) does not ensure proximity in either
the immediate or terminal downstream attribute.
These observations indicate that \(\tau_{35}\) merges low-level states
whose discarded information remains relevant to subsequent attribute
evolution, providing empirical evidence of an AIC violation. 

\begin{figure*}[t]
    \centering
    \begin{minipage}[t]{0.48\textwidth}
        \centering
        \includegraphics[width=\linewidth]
        {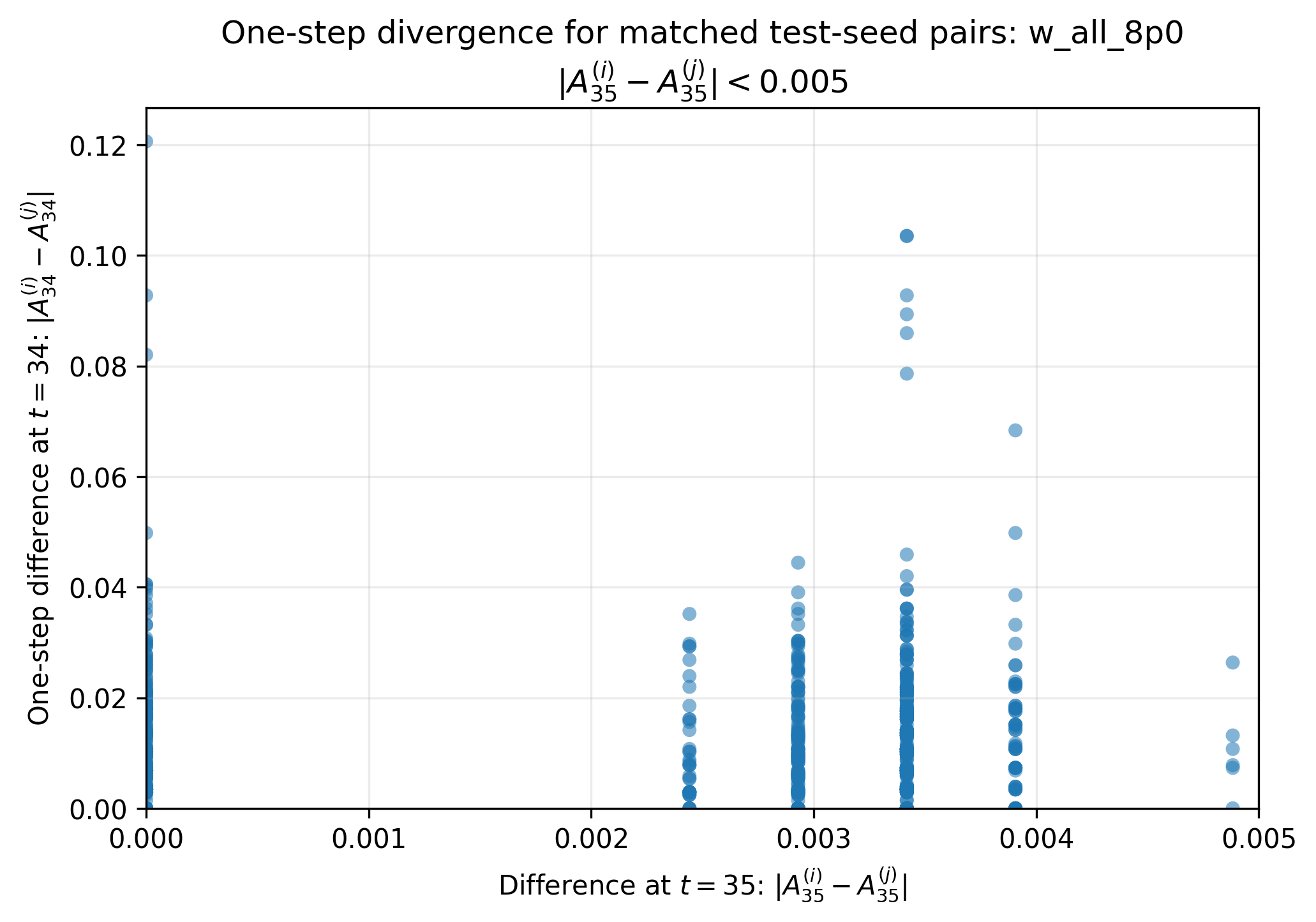}\\
        \textbf{(a) One-step divergence at \(t=34\)}
    \end{minipage}
    \hfill
    \begin{minipage}[t]{0.48\textwidth}
        \centering
        \includegraphics[width=\linewidth]
        {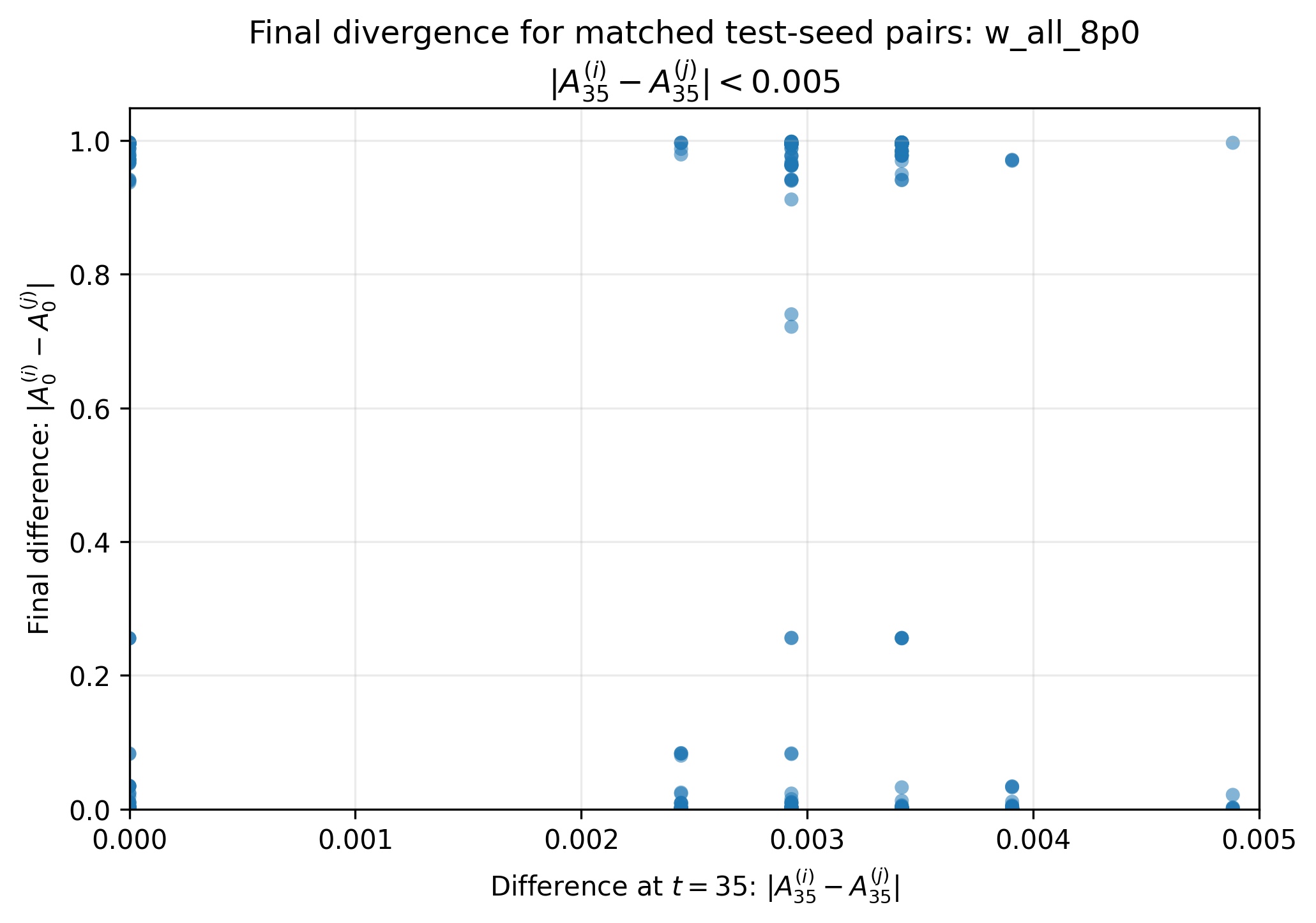}\\
        \textbf{(b) Final-outcome divergence at \(t=0\)}
    \end{minipage}
    \caption{Downstream divergence between matched held-out trajectory
    pairs at \(w=8\). Pairs are selected such that
    \(\lvert a_{35}^{(i)}-a_{35}^{(j)}\rvert<0.005\). Despite their
    nearly identical abstract values at \(t=35\), the pairs can exhibit
    substantial differences (a) at the next denoising step and (b) in
    the final attribute.}
    \label{fig:supp-aic-violation}
\end{figure*}

\section{Additional Experimental Results}
We provide distribution-level comparisons between the final
perceived-masculine scores generated by the low-level Stable Diffusion
v1.5 model and the high-level causal abstraction in Figures 6,7,8. These results
complement the expected-score comparisons in the main paper by showing
how well the abstraction preserves the complete output distribution.

Overall, the comparisons show that the high-level model generally
preserves the prompt-specific shapes and guidance-dependent changes of the low-level attribute distributions.

\begin{figure*}[t]
    \centering
    \includegraphics[width=0.95\textwidth]
     {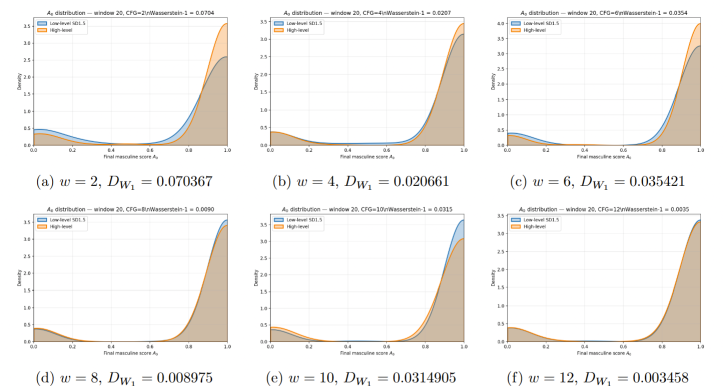}
    
    \caption{Final perceived-masculine score distributions produced by
    low-level SD1.5 and the high-level abstraction for the lawyer
    prompt across CFG scales. Each panel reports the corresponding
    Wasserstein-1 distance.}
    \label{fig:supp-lawyer-distributions}
\end{figure*}

\begin{figure*}[t]
    \centering
    \includegraphics[width=0.95\textwidth]
    {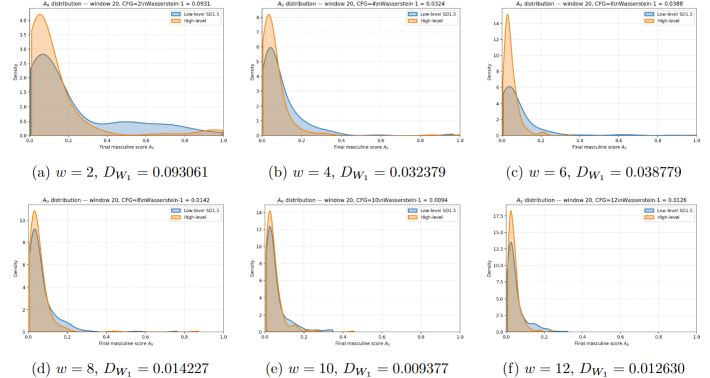}
    
    \caption{Final perceived-masculine score distributions produced by
    low-level SD1.5 and the high-level abstraction for the librarian
    prompt across CFG scales. Each panel reports the corresponding
    Wasserstein-1 distance.}
    \label{fig:supp-librarian-distributions}
\end{figure*}

\begin{figure*}[t]
    \centering
     \includegraphics[width=0.95\textwidth]
     {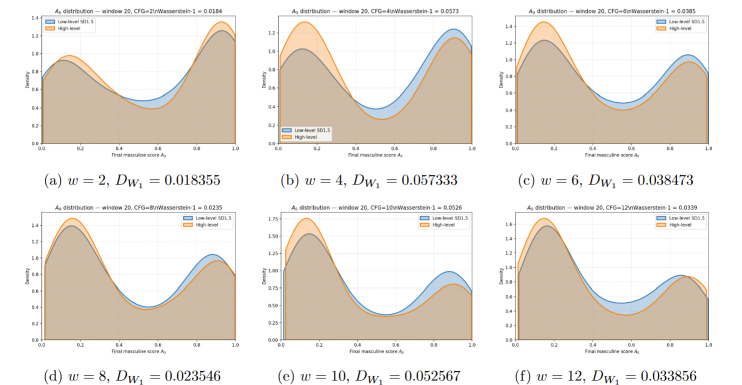}
    
    \caption{Final perceived-masculine score distributions produced by
    low-level SD1.5 and the high-level abstraction for the scientist
    prompt across CFG scales. Each panel reports the corresponding
    Wasserstein-1 distance.}
    \label{fig:supp-scientist-distributions}
\end{figure*}

\end{document}